%% file: main_prime.tex
\documentclass{article}

\usepackage{PRIMEarxiv}
\usepackage[utf8]{inputenc}
\usepackage[T1]{fontenc}
\usepackage{amsmath,amssymb,amsfonts,amsthm,mathtools}
\usepackage{algorithm}
\usepackage{algpseudocode}
\usepackage{array,booktabs}
\usepackage{enumitem}
\usepackage{graphicx}
\usepackage{caption}
\usepackage{microtype}
\usepackage{nicefrac}
\usepackage{url}
\usepackage{xcolor}
\usepackage{hyperref}
\graphicspath{{./}{media/}}

\numberwithin{equation}{section}
\setlist[itemize]{nosep,leftmargin=*,topsep=0pt,partopsep=0pt}
\setlist[enumerate]{leftmargin=*}

\newtheorem{theorem}{Theorem}[section]
\newtheorem{lemma}[theorem]{Lemma}
\newtheorem{proposition}[theorem]{Proposition}
\newtheorem{corollary}[theorem]{Corollary}

\theoremstyle{definition}
\newtheorem{definition}[theorem]{Definition}

\newtheorem{example}[theorem]{Example}

\theoremstyle{remark}
\newtheorem{remark}[theorem]{Remark}

\title{Exit--and--Join Dynamics for Decentralized Coalition Formation}
\author{
  Quanyan Zhu \\
  Department of Electrical and Computer Engineering \\
  New York University Tandon School of Engineering \\
  Brooklyn, NY, USA \\
  \texttt{quanyan.zhu@nyu.edu}
}

\begin{document}
\maketitle

\begin{abstract}
This paper studies coalition formation as a decentralized dynamical process
driven by unilateral exit--and--join decisions. Agents evaluate local moves using
the Aumann--Dr\`eze value, so payoffs are computed within the agent's current
coalition rather than through a globally negotiated coalition structure. The
resulting model links cooperative payoff allocation with noncooperative
best-response behavior: a terminal partition is precisely a coalition structure
with no admissible, individually profitable exit--and--join deviation. We
establish equilibrium characterizations, identify conditions under which the
dynamics admit scalar Lyapunov or exact-potential representations, and analyze
how switching and acceptance costs shape local stability. Numerical experiments
test finite-time stabilization, cost sensitivity, and a special convex-game
benchmark.
\end{abstract}

\keywords{coalition formation \and Aumann--Dr\`eze value \and exit--and--join dynamics \and potential games \and Lyapunov analysis \and multi-agent systems}

\input{join-exit}

\bibliographystyle{abbrv}
\bibliography{refs}

\end{document}

%% file: join-exit.tex
\section{Introduction}

In many economic, social, and socio--technical systems, agents
continuously move across organizational boundaries. Individuals leave existing
groups, join new ones, or form new organizations altogether. These transitions
are not orchestrated by a central authority; rather, they arise from a large
number of decentralized decisions made by self-interested agents. As a result,
coalition structures are not designed ex ante, but emerge endogenously from
ongoing exit--and--join behavior, as in endogenous and dynamic coalition
formation models
\cite{HartKurz1983,RayVohra1999,konishi2003coalition,ZhuHan2026SplitMerge}.
Despite this reality, much of the existing literature models coalition formation
as a top--down process. Classical split--and--merge algorithms, hierarchical
clustering methods, and centralized optimization approaches typically assume
global information, coordinated decision making, or a planner capable of
evaluating alternative coalition structures at the system level
\cite{SandholmEtAl1999,ChalkiadakisElkindWooldridge2011}. While such models are
analytically convenient, they obscure the distributed nature of coalition
formation observed in practice.

This paper adopts a fundamentally different perspective. We study coalition
formation as a decentralized dynamical process driven by local exit--and--join
decisions. Each agent evaluates whether to remain in its current coalition, join
another coalition, or form a singleton based solely on local payoff comparisons.
No agent requires knowledge of the global coalition structure, and no centralized
coordination or negotiation is assumed. Coalition structures therefore arise as
the outcome of distributed individual decisions rather than as the solution to a
centralized design problem. This local-decision viewpoint is close in spirit to
hedonic coalition formation, where each agent evaluates coalitions through the
members of its own group \cite{DrezeGreenberg1980,BogomolnaiaJackson2002}.
Extensions to settings with limited information, for example through learning or
belief updating, fit naturally within this framework and are left for future
work \cite{chalkiadakis2004bayesian,hamed2023distributed}.

Within this decentralized process, coalitions may dissolve when they no longer
generate sufficient value for their members, while new coalitions may emerge as
agents split off and reorganize. The resulting patterns of formation, dissolution,
and reconfiguration closely mirror those observed in firms, alliances, online
communities, and modular socio--technical systems.
From a game--theoretic standpoint, coalition formation in this setting lies at
the intersection of noncooperative and cooperative games. Agents behave
noncooperatively in the sense that they pursue individually rational improvements
through unilateral exit--and--join moves. At the same time, payoffs within each
coalition are determined cooperatively through solution concepts such as the
Aumann--Dr\`eze or Owen values. The coalition structure itself is thus an emergent
object, shaped by individual incentives interacting with cooperative payoff
allocation rules \cite{AumannDreze1974,owen1977values,Owen2013}.

\subsection{Relevant Applications}

In multi--agent artificial intelligence systems, agents must decide which tasks,
projects, or workflows to participate in. An agent may initially join a
collaborative task based on partial or local information and subsequently exit
when a higher--value or better--matched task becomes available. Such decisions are
typically made autonomously and without centralized coordination, relying instead
on local assessments of expected utility, resource requirements, or compatibility
with other agents. The resulting allocation of agents to tasks therefore emerges
from distributed exit--and--join decisions rather than from a global planning
mechanism. In this interpretation, tasks correspond to coalitions, and acceptance
reflects task capacity, compatibility, or coordination constraints
\cite{ShehoryKraus1998,SandholmEtAl1999}.

Labor markets provide a canonical example of exit--and--join dynamics. Workers
leave existing organizations and join new ones when they expect higher
compensation, improved working conditions, or better career prospects. These
decisions are decentralized and often myopic, based on local information such as
job offers, peer outcomes, and internal organizational policies. Firms impose
acceptance constraints through hiring decisions, compensation structures, and
promotion rules. Public policies and institutional factors, including hiring
frictions, noncompete clauses, and labor regulations, directly affect switching
costs and acceptance conditions, thereby shaping patterns of workforce mobility
and talent flow. The resulting organizational structure of the labor market
emerges endogenously from individual incentives interacting with these
constraints, paralleling coalition-formation models in which admissible
partitions are shaped by local incentives and institutional rules
\cite{HartKurz1983,BanerjeeKonishiSonmez2001}.

Subscription markets, such as mobile service providers or digital streaming
platforms, exhibit similar decentralized reconfiguration dynamics. Consumers
periodically reassess their memberships and may exit one service to join another
based on price, quality, bundled offerings, or network effects. Service providers
influence acceptance indirectly through pricing tiers, contract terms, and service
differentiation. From a modeling perspective, platforms correspond to competing
coalitions, while consumers act as agents who repeatedly evaluate whether to
remain or switch. Market structure, including user distribution and churn rates,
thus arises from many individual exit--and--join decisions rather than from
centralized coordination.

Across these applications, coalition formation is driven by distributed individual decisions, shaped by local payoff comparisons and acceptance constraints, rather than by global optimization. Coalitions form, grow, shrink, or disappear as agents respond to changing incentives, costs, and opportunities. This common structure makes exit--and--join dynamics a unifying modeling framework for organizational formation and reconfiguration in both human and artificial systems.
The paper makes four contributions. First, it formulates coalition formation as
an asynchronous exit--and--join process in which agents use only coalition-local
Aumann--Dr\`eze payoffs. Second, it characterizes terminal coalition structures
as fixed points of an induced noncooperative best-response problem. Third, it
separates efficiency properties of the underlying cooperative game from dynamic
implementability, showing that scalar Lyapunov and exact-potential
representations require explicit incentive-alignment conditions. Fourth, it
uses numerical experiments to examine finite termination, switching and
acceptance frictions, and a special convex benchmark in which the grand
coalition is dynamically selected.

\section{Related Work}
\label{sec:related_work}

Coalition formation has a long history in cooperative and noncooperative game
theory. Classical cooperative analysis studies stability and efficiency of
coalitions through solution concepts such as the core and the Shapley value
\cite{gillies1953core,Shapley1953,Maschler2013,Myerson1991,Moulin1988}, while
endogenous coalition formation models study how partitions arise from strategic
behavior rather than being imposed exogenously. Early and influential treatments
include coalition formation through game forms and binding agreements
\cite{HartKurz1983,RayVohra1999}, sequential coalition formation with
externalities \cite{Bloch1996Sequential}, dynamic coalition formation processes
\cite{konishi2003coalition}, split--merge dynamics for Shapley-fair coalition
formation \cite{ZhuHan2026SplitMerge}, and dynamic cooperative games
\cite{filar2000dynamic,bauso2009robust}. The present paper is closest in spirit
to this dynamic literature, but it differs by using the Aumann--Dr\`eze value as
the payoff rule at every intermediate partition and by focusing on local
exit--and--join deviations with acceptance and switching costs.

A second related line of work studies hedonic coalition formation, where each
agent's preferences depend only on the members of its own coalition. Hedonic
coalitions were introduced by Dr\`eze and Greenberg \cite{DrezeGreenberg1980},
and stability notions for hedonic partitions were developed extensively in later
work, including Bogomolnaia and Jackson's analysis of hedonic coalition
structures \cite{BogomolnaiaJackson2002}. Subsequent work has studied stable
partitions and generic coalition-formation procedures
\cite{AptRadzik2009,apt2009coalition}, as well as simple models of core
stability under unilateral or group deviations \cite{BanerjeeKonishiSonmez2001}.
Our model shares the locality of hedonic games because an agent's payoff is
computed from its current coalition. However, preferences are not primitive:
they are induced endogenously by a transferable-utility game and a cooperative
allocation rule, which makes it possible to connect local incentives to surplus
and potential functions.

The allocation rule used here builds directly on cooperative values with
coalition structures. The Aumann--Dr\`eze value allocates worth within each
coalition of a fixed partition \cite{AumannDreze1974}, whereas the Owen value
first treats coalitions as a priori unions in a quotient game and then allocates
within each union \cite{owen1977values,Owen2013}. Related cooperative models
also study restrictions induced by networks or communication structures, as in
Myerson's graph-restricted games \cite{Myerson1977Graphs}, and strategic
network formation models study how links and cooperation structures arise from
individual incentives \cite{JacksonWolinsky1996}. These approaches provide
alternative ways of embedding organizational constraints into payoff allocation.
The present paper uses the Aumann--Dr\`eze value because its
coalition-local structure is compatible with decentralized exit--and--join
updates and with local acceptance tests performed by destination coalitions.

There is also a substantial computational and multi-agent literature on
coalition formation. Coalition-structure generation algorithms typically seek a
partition that maximizes total value or provides approximation guarantees under
large search spaces \cite{SandholmEtAl1999,ChalkiadakisElkindWooldridge2011}.
Task-allocation methods for multi-agent systems use coalition formation to
assign agents to tasks under resource and compatibility constraints
\cite{ShehoryKraus1998}, coalitional control uses cooperative-game tools for
networked control and coordination \cite{fele2017coalitional}, and
learning-based approaches study coalition formation under uncertainty
\cite{chalkiadakis2004bayesian}. By contrast, the exit--and--join dynamics
studied here do not solve a centralized
coalition-structure generation problem. They describe a decentralized adjustment
process in which agents move myopically, destination coalitions apply local
acceptance constraints, and convergence is established through Lyapunov or
potential arguments rather than through global optimization.

Finally, the convergence analysis is related to potential games and learning
dynamics. Potential games show how unilateral incentives can be represented by a
single scalar function \cite{MondererShapley1996}, and cooperative-game work has
connected values, potentials, and consistency conditions
\cite{hart1989potential,hart1996bargaining}. Recent work on learning in
coalitional games also emphasizes distributed adaptation and convergence
\cite{hamed2023distributed,smyrnakis2019game}. The contribution here is to
identify when Aumann--Dr\`eze exit--and--join incentives admit exact or ordinal
alignment with coalition surplus. This separation clarifies why cooperative
convexity can imply efficiency of the grand coalition without guaranteeing that
myopic decentralized dynamics will select it.

\section{Aumann--Dr\`eze Value}
\label{sec:ad_value}

\subsection{Motivation}
\label{subsec:motivation}

Let $N=\{1,\dots,n\}$ be a finite set of agents and let
$v:2^N\to\mathbb{R}$ be a transferable--utility (TU) cooperative game with
$v(\varnothing)=0$.
The Shapley value $\phi(v)$ \cite{Shapley1953} presumes that the grand coalition
$N$ forms and that the entire surplus $v(N)$ is distributed among agents according
to expected marginal contributions averaged over all permutations of players.
Under this assumption, all potential synergies encoded in the characteristic
function $v$ are fully realizable, regardless of any intermediate organizational,
institutional, or technological constraints.

In many socio--technical, economic, and organizational systems, however,
cooperation is constrained by an existing \emph{coalition structure}
$T=\{T_1,\dots,T_m\}\in\Pi(N)$, which specifies the coalitions within which
binding agreements, transfers, and coordination are feasible.
Agents belonging to different coalitions in $T$ cannot form binding
subcoalitions or share surplus directly, even if the characteristic function $v$
assigns positive value to such cross--coalition cooperation.
Examples include firms with rigid departmental boundaries, political alliances,
supply--chain consortia, or modular cyber--physical systems in which coordination
across subsystems is costly or infeasible.
The Aumann--Dr\`eze (AD) value \cite{AumannDreze1974} is designed precisely for
this setting.
It provides a Shapley--consistent allocation rule that respects the constraints
imposed by a given coalition structure.
Rather than assuming that the grand coalition forms, the AD value conditions on
the coalition structure $T$ and allocates surplus \emph{within} each coalition.

Conceptually, the Aumann--Dr\`eze value can be understood as a two--stage
construction with a \emph{fixed} first stage.
In the first stage, the coalition structure $T$ is taken as exogenously given and
no bargaining or strategic interaction takes place across coalitions.
In the second stage, each coalition $T_j$ independently distributes its worth
$v(T_j)$ among its members according to the Shapley value of the restricted game
$v_{|T_j}$.
Thus, while marginal contributions are averaged over permutations of players
within each coalition, no marginal contributions across distinct coalitions are
ever considered.
As a consequence, surplus generated outside a coalition is not shared, and
potential cross--coalition synergies encoded in $v$ remain unrealized.
The Aumann--Dr\`eze allocation therefore generally differs from the Shapley value
whenever the characteristic function exhibits complementarities across distinct
coalitions.

The Owen value \cite{owen1977values,Owen2013} also admits a two--stage interpretation, but of a
fundamentally different nature.
In the first stage, coalitions in $T$ bargain strategically as aggregate players
in a \emph{quotient game} induced by the original characteristic function $v$.
This stage explicitly internalizes cross--coalition externalities at the level of
coalitions.
In the second stage, the surplus allocated to each coalition is distributed
internally among its members according to the Shapley value.
The Owen value coincides with the Aumann--Dr\`eze value only in special cases,
most notably when the game is \emph{additively separable across the coalition
structure} $T$, so that no strategic externalities arise between distinct
coalitions \cite{AumannDreze1974,owen1977values}.
In general, however, the Owen value internalizes cross--coalition interactions at
the coalition level, whereas the Aumann--Dr\`eze value does not.

From a dynamic perspective, the Aumann--Dr\`eze value is particularly well suited
to modeling decentralized coalition formation and adaptation.
Because payoffs depend only on local coalition membership, agents can evaluate
unilateral exit--and--join deviations without requiring global coordination,
coalition--level renegotiation, or computation of a quotient game.
This locality property makes the Aumann--Dr\`eze value a natural foundation for
the exit--and--join dynamics studied in this paper, where coalition structures
evolve endogenously through myopic, individually rational moves subject to
acceptance and switching costs.

\subsection{Definition}
\label{subsec:ad_definition}

Let $N=\{1,\dots,n\}$ be a finite set of players and let
$v:2^N\to\mathbb{R}$ be a transferable--utility cooperative game with
$v(\varnothing)=0$.
Let $\Pi(N)$ denote the set of all partitions of $N$. For any coalition structure $T=\{T_1,\dots,T_m\}\in\Pi(N)$ and any player
$i\in N$, denote by $C_T(i)\in T$ the unique coalition containing $i$.
Given $T$ and $i$, the restriction of $v$ to $C_T(i)$ is defined by
\begin{equation}
\label{eq:restricted_game}
v_{|C_T(i)}(S):=v(S),
\qquad
\forall\, S\subseteq C_T(i).
\end{equation}

\begin{example}[Restriction of the game to a coalition]
\label{ex:restricted_game}

Let $N=\{1,2,3\}$ and consider the transferable--utility game
$v:2^N\to\mathbb{R}$ defined by the coalition values in
Table~\ref{tab:game_v_values}.

\begin{table}[t]
\centering
\caption{Characteristic function values of the TU game $v$}
\label{tab:game_v_values}
\begin{tabular}{c|c}
\hline
Coalition $S$ & $v(S)$ \\
\hline
$\varnothing$ & $0$ \\
$\{1\}$ & $0$ \\
$\{2\}$ & $0$ \\
$\{3\}$ & $0$ \\
$\{1,2\}$ & $4$ \\
$\{1,3\}$ & $2$ \\
$\{2,3\}$ & $0$ \\
$\{1,2,3\}$ & $6$ \\
\hline
\end{tabular}
\end{table}

Consider the coalition structure $T=\{\{1,2\},\{3\}\}\in\Pi(N)$. Under the
partition $T$, the restricted games are obtained by restricting the domain of
$v$ to subsets of each coalition.

\medskip
\noindent\textbf{Restricted game on $\{1,2\}$.}
\begin{center}
\begin{tabular}{c|c}
\hline
$S\subseteq\{1,2\}$ & $v_{|\{1,2\}}(S)$ \\
\hline
$\varnothing$ & $0$ \\
$\{1\}$ & $0$ \\
$\{2\}$ & $0$ \\
$\{1,2\}$ & $4$ \\
\hline
\end{tabular}
\end{center}

\medskip
\noindent\textbf{Restricted game on $\{3\}$.}
\begin{center}
\begin{tabular}{c|c}
\hline
$S\subseteq\{3\}$ & $v_{|\{3\}}(S)$ \\
\hline
$\varnothing$ & $0$ \\
$\{3\}$ & $0$ \\
\hline
\end{tabular}
\end{center}

Thus, relative to the coalition structure $T$, the original game $v$ decomposes
into two independent subgames: a two-player game on $\{1,2\}$ and a singleton
game on $\{3\}$.
\end{example}

\begin{definition}[Aumann--Dr\`eze value]
\label{def:ad_value}

The \emph{Aumann--Dr\`eze value} of $v$ relative to the coalition structure $T$ is
the payoff vector $\Omega(v;T)=\bigl(\Omega_i(v;T)\bigr)_{i\in N}$ defined
componentwise by
\begin{equation}
\label{eq:ad_definition}
\Omega_i(v;T)
=
\phi_i\!\left(v_{|C_T(i)}\right),
\qquad
i\in N,
\end{equation}
where $\phi(\cdot)$ denotes the Shapley value.
\end{definition}

Equivalently, for any coalition $T_j\in T$ and any player $i\in T_j$,
\begin{equation}
\label{eq:ad_formula}
\Omega_i(v;T)
=
\sum_{S\subseteq T_j\setminus\{i\}}
\frac{|S|!\,(|T_j|-|S|-1)!}{|T_j|!}
\bigl[v(S\cup\{i\})-v(S)\bigr].
\end{equation}

\begin{example}[Computation of the Aumann--Dr\`eze value]
\label{ex:ad_computation}

Using \eqref{eq:ad_formula} for the partition
$T=\{\{1,2\},\{3\}\}$, we compute the Aumann--Dr\`eze value.

For agent $1$,
\[
\Omega_1(v;T)
=
\sum_{S\subseteq\{2\}}
\frac{|S|!\,(2-|S|-1)!}{2!}
\bigl[v(S\cup\{1\})-v(S)\bigr],
\]
which yields
\begin{center}
\begin{tabular}{c c c c}
\hline
$S$ & Weight & $v(S\cup\{1\})-v(S)$ & Contribution \\
\hline
$\varnothing$ & $\frac{1}{2}$ & $0$ & $0$ \\
$\{2\}$ & $\frac{1}{2}$ & $4$ & $2$ \\
\hline
\end{tabular}
\end{center}
Hence $\Omega_1(v;T)=2$. By symmetry, $\Omega_2(v;T)=2$.
Since $C_T(3)=\{3\}$ is a singleton coalition, $\Omega_3(v;T)=0$.
Thus $\Omega(v;T)=(2,2,0)$.
\end{example}

\begin{proposition}[Coalitional efficiency]
\label{prop:ad_efficiency}

For every coalition $T_j\in T$,
\begin{equation}
\label{eq:ad_efficiency}
\sum_{i\in T_j}\Omega_i(v;T)=v(T_j).
\end{equation}
\end{proposition}

\begin{proof}
For each $T_j\in T$, the vector $(\Omega_i(v;T))_{i\in T_j}$ coincides with the
Shapley value of the restricted game $v_{|T_j}$. Since the Shapley value is
efficient, the sum equals $v(T_j)$.
\end{proof}

\begin{example}[Coalitional efficiency in the running example]
\label{ex:ad_efficiency}

For the coalition $\{1,2\}$,
$\Omega_1(v;T)+\Omega_2(v;T)=2+2=4=v(\{1,2\})$, and for the singleton
coalition $\{3\}$, $\Omega_3(v;T)=0=v(\{3\})$.
Thus \eqref{eq:ad_efficiency} holds for every coalition in $T$.
\end{example}

\section{Exit--and--Join Rules under Aumann--Dr\`eze}
\label{sec:exit_join_ad}

\subsection{State space and notation}

Let $N=\{1,\dots,n\}$ be a finite set of agents and let
$v:2^N\to\mathbb{R}$ be a transferable--utility (TU) cooperative game with
$v(\varnothing)=0$.
The system state at time $t$ is a coalition structure
$T_t \in \Pi(N)$, where $\Pi(N)$ denotes the set of all partitions of $N$.
For each agent $i\in N$, let $C_{T_t}(i)\in T_t$ denote the unique coalition
containing $i$ at time $t$.
Payoffs are evaluated using the Aumann--Dr\`eze value:
$\Omega_i(v;T_t)=\phi_i\!\left(v\big|_{C_{T_t}(i)}\right)$, that is, agent
$i$ receives its Shapley value in the game restricted to its current coalition.
The evolution of the system is driven by unilateral exit--and--join decisions of
individual agents.

\subsection{Exit--and--join actions}

\begin{definition}[Exit--and--join action]
\label{def:exit_join_action}
Let $T_t\in\Pi(N)$ be a coalition structure at time $t$ and let $i\in N$ be an
agent.
An \emph{exit--and--join action} of agent $i$ at time $t$ consists of the
selection of a destination
\[
D\in\mathcal D_i(T_t)
:=
\bigl(T_t \setminus \{C_{T_t}(i)\}\bigr)\cup\{\varnothing\},
\]
where $C_{T_t}(i)$ denotes the unique coalition of $T_t$ containing agent $i$.
\end{definition}

\begin{definition}[Exit--and--join transition]
\label{def:exit_join_transition}
Given a coalition structure $T_t\in\Pi(N)$, an agent $i\in N$, and a destination
$D\in\mathcal D_i(T_t)$, the \emph{exit--and--join transition} induced by agent $i$
and destination $D$ is the coalition structure
\begin{equation}
\label{eq:ej_transition}
T_{t+1}
=
F_i(T_t,D)
:=
\bigl(T_t\setminus\{C_{T_t}(i),D\}\bigr)
\cup
\bigl\{C_{T_t}(i)\setminus\{i\}\bigr\}^{\!*}
\cup
\{D\cup\{i\}\},
\end{equation}
where $\{\cdot\}^{\!*}$ denotes inclusion only if the set is nonempty, and where
$D\cup\{i\}=\{i\}$ when $D=\varnothing$.
\end{definition}

The mapping $F_i:\Pi(N)\times\mathcal D_i(T_t)\to\Pi(N)$ defines a well--posed
discrete--time state transition on the space of coalition structures.
If $D=\varnothing$, the transition corresponds to agent $i$ exiting its current
coalition and forming a singleton.
If $D\in T_t\setminus\{C_{T_t}(i)\}$, the transition corresponds to agent $i$
exiting its current coalition and joining the destination coalition.
If agent $i$ is initially a singleton and joins a nonempty destination, the
singleton coalition disappears.
In all cases, the transition modifies at most two coalitions and preserves the
partition structure of $\Pi(N)$.

\subsection{Acceptance rule}

Exit--and--join moves may be subject to acceptance by the destination coalition.

\begin{definition}[Acceptance rule]
\label{def:acceptance}
Given a proposed move $T_t \to F_i(T_t,D)$, the destination coalition
$D\in T_t$ accepts agent $i$ if
\[
\Omega_\ell(v;F_i(T_t,D))
\ge
\Omega_\ell(v;T_t),
\qquad
\forall \ell\in D.
\]
\end{definition}
A costly-admission variant is useful when incumbents incur onboarding,
coordination, or congestion costs from admitting a new member. In that case a
scalar acceptance cost $\kappa\ge 0$ modifies the rule to
\(\Omega_\ell(v;F_i(T_t,D))-\kappa\ge\Omega_\ell(v;T_t)\) for every
\(\ell\in D\). The baseline acceptance rule corresponds to \(\kappa=0\), while
larger values of \(\kappa\) make coalitions more selective even when entry would
not lower their gross Aumann--Dr\`eze payoffs.
A particularly simple benchmark is \emph{automatic consent}, in which all
destination coalitions accept all entrants.
This corresponds to the case in which coalition membership is unrestricted
and agents are free to join any coalition unilaterally.
Formally, the acceptance condition is vacuous and every proposed exit--and--join
move is admissible subject only to the deviating agent's own incentive.

\subsection{Switching costs}

While exit--and--join actions describe the set of feasible coalition
reconfigurations, actual coalition changes may involve frictions that affect an
agent's willingness to deviate.
Such frictions may arise from coordination effort, renegotiation of internal
transfers, loss of institutional capital, or adjustment delays associated with
entering a new coalition.
To capture these effects, we associate a switching cost with each exit--and--join
transition.

\begin{definition}[Switching cost]
\label{def:switching_cost}
For any coalition structure $T_t\in\Pi(N)$, agent $i\in N$, and destination
$D\in\mathcal D_i(T_t)$, the switching cost incurred by agent $i$ when executing
the exit--and--join transition $T_t \longrightarrow F_i(T_t,D)$ is given by a
function $c_i\bigl(T_t,F_i(T_t,D)\bigr) \ge 0$.
\end{definition}

The switching cost is allowed to depend on the current coalition structure, the
identity of the moving agent, and the resulting coalition structure, thereby
accommodating heterogeneous frictions and history--dependent adjustment costs.
A fundamental benchmark is the frictionless case in which switching is costless.
Throughout much of the analysis, we therefore consider the zero--cost regime
defined by
\[
c_i\bigl(T_t,F_i(T_t,D)\bigr)\equiv 0
\qquad
\text{for all } i\in N,\; T_t\in\Pi(N),\; D\in\mathcal D_i(T_t).
\]
This regime isolates the incentive effects induced purely by the Aumann--Dr\`eze
payoff and yields the simplest form of exit--and--join dynamics.
Positive switching costs will be introduced subsequently to refine equilibrium
selection and to strengthen stability and convergence properties of the induced
dynamical system.

\subsection{Decision rule}

Given a coalition structure $T_t\in\Pi(N)$ at time $t$, each agent evaluates
potential exit--and--join actions by comparing its payoff under the current
coalition structure with the payoff induced by the corresponding transition,
net of any switching costs.
The evaluation is unilateral in the sense that each agent treats the induced
coalition structure as fixed when assessing the consequences of a deviation.

\begin{definition}[Aumann--Dr\`eze exit--and--join rule]
\label{def:ad_ej_rule}
At time $t$, agent $i\in N$ accepts a destination
$D\in\mathcal D_i(T_t)$ if the induced exit--and--join transition
$T_t\to F_i(T_t,D)$ satisfies
\[
\Omega_i\bigl(v;F_i(T_t,D)\bigr)
-
c_i\bigl(T_t,F_i(T_t,D)\bigr)
>
\Omega_i(v;T_t),
\]
and the destination coalition $D$ satisfies the acceptance condition.
If multiple destinations satisfy this inequality, agent $i$ selects a
destination that maximizes its net payoff
$\Omega_i(v;F_i(T_t,D)) - c_i(T_t,F_i(T_t,D))$.
If no destination is admissible, the agent remains in its current coalition and
the state remains unchanged, so that $T_{t+1}=T_t$.
\end{definition}

The resulting behavior is myopic in the sense that agents optimize their own
Aumann--Dr\`eze payoff at each step without anticipating future coalition
reconfigurations or strategic responses by other agents.
Consequently, the exit--and--join rule induces a discrete--time, asynchronous
best--response dynamics on the finite state space $\Pi(N)$.

\section{Exit--and--Join Equilibrium}

We now formalize the equilibrium notion induced by unilateral exit--and--join
behavior under the Aumann--Dr\`eze value and examine its structural properties.

\begin{definition}[Exit--and--join equilibrium]
\label{def:exit_join_equilibrium}
A coalition structure $T^\star \in \Pi(N)$ is an
\emph{exit--and--join equilibrium} under the Aumann--Dr\`eze value if there exists
no agent $i \in N$ and no destination $D \in \mathcal D_i(T^\star)$ such that
the exit--and--join transition $T^\star \to F_i(T^\star,D)$ is admissible and
\[
\Omega_i\!\left(v;F_i(T^\star,D)\right)
-
c_i\!\left(T^\star,F_i(T^\star,D)\right)
>
\Omega_i(v;T^\star).
\]
\end{definition}

Equivalently, no agent can strictly improve its Aumann--Dr\`eze payoff by a
unilateral exit--and--join deviation that is accepted by the destination
coalition.

\begin{example}[Exit--and--join deviation in the running example]
\label{ex:running_exit_join}
We continue with the transferable--utility game, coalition structure, and
Aumann--Dr\`eze payoff vector introduced in
Examples~\ref{ex:restricted_game}, \ref{ex:ad_computation},
and~\ref{ex:ad_efficiency}.
At time $t$, the coalition structure is
$T_t=\bigl\{\{1,2\},\{3\}\bigr\}$, with associated Aumann--Dr\`eze payoffs
$\Omega(v;T_t)=(2,2,0)$. We examine unilateral exit--and--join actions under
zero switching cost and automatic acceptance.

Consider first agent $2$.
Since $C_{T_t}(2)=\{1,2\}$, the feasible destinations are
$\mathcal D_2(T_t)=\{\{3\},\varnothing\}$.
Choosing $D=\{3\}$ yields
$T_{t+1}=F_2(T_t,\{3\})=\bigl\{\{1\},\{2,3\}\bigr\}$.
The restricted game on $\{2,3\}$ has zero worth, hence
\[
\Omega_2\bigl(v;F_2(T_t,\{3\})\bigr)=0
<
\Omega_2(v;T_t)=2.
\]
Thus this deviation is not profitable.
Next consider agent $3$.
Since $C_{T_t}(3)=\{3\}$, agent $3$ may choose $D=\{1,2\}$, yielding
$T_{t+1}=F_3(T_t,\{1,2\})=\bigl\{\{1,2,3\}\bigr\}$.
For this coalition structure, the Aumann--Dr\`eze value coincides with the Shapley
value of the grand coalition, and therefore
\[
\Omega_3\bigl(v;F_3(T_t,\{1,2\})\bigr)
=
\phi_3(v)
>
\Omega_3(v;T_t)=0.
\]
Hence agent $3$ admits a profitable exit--and--join deviation, and
$T_t=\{\{1,2\},\{3\}\}$ is not an exit--and--join equilibrium.
\end{example}

\subsection{Best--Response and Nash Equilibrium Interpretation}

Exit--and--join equilibrium admits a precise noncooperative interpretation as a
best--response fixed point and, equivalently, as a pure--strategy Nash equilibrium
of an induced, state--dependent game defined at each coalition structure.

\begin{definition}[Induced exit--and--join game]
\label{def:induced_ej_game}
Fix a coalition structure $T \in \Pi(N)$.
The \emph{induced exit--and--join game} at $T$ is the noncooperative game
\[
\mathcal G(T)
=
\bigl(
N,\{A_i(T)\}_{i\in N},\{u_i(\cdot;T)\}_{i\in N}
\bigr),
\]
where the player set is $N$. The action set of player $i$ is the
accepted-action set
\[
A_i(T)
:=
\{C_T(i)\}
\cup
\Bigl\{
D\in\mathcal D_i(T):
D=\varnothing \text{ or } D \text{ accepts } i
\Bigr\};
\]
choosing $D=C_T(i)$ is the stay action, with $F_i(T,C_T(i)):=T$ and
$c_i(T,T):=0$. The payoff to player $i$ from choosing $D \in A_i(T)$ is
\[
u_i(D;T)
:=
\Omega_i\bigl(v;F_i(T,D)\bigr)
-
c_i\bigl(T,F_i(T,D)\bigr).
\]
Actions chosen by other agents are not jointly implemented and therefore do not
enter the payoff function.
\end{definition}

The best--response correspondence of agent $i$ at coalition structure $T$ is
\[
\mathrm{BR}_i(T)
:=
\arg\max_{D \in A_i(T)} u_i(D;T).
\]

\begin{lemma}[Best--response characterization]
\label{lem:best_response}
A coalition structure $T^\star$ is an exit--and--join equilibrium if and only if,
for every agent $i \in N$,
\[
C_{T^\star}(i) \in \mathrm{BR}_i(T^\star).
\]
\end{lemma}

\begin{proof}
($\Rightarrow$)
Suppose that $T^\star$ is an exit--and--join equilibrium.
By Definition~\ref{def:exit_join_equilibrium}, there exists no destination
$D \in \mathcal D_i(T^\star)$ such that the exit--and--join transition
$T^\star \to F_i(T^\star,D)$ is admissible and yields a strictly higher payoff for
agent $i$.
Equivalently,
\[
u_i(D;T^\star)
\le
u_i\bigl(C_{T^\star}(i);T^\star\bigr),
\qquad
\forall\, D \in A_i(T^\star).
\]
Since choosing $D=C_{T^\star}(i)$ leaves the coalition structure unchanged and
incurs zero switching cost, it attains the maximum of $u_i(\cdot;T^\star)$.
Hence $C_{T^\star}(i) \in \mathrm{BR}_i(T^\star)$.

($\Leftarrow$)
Conversely, suppose that for every agent $i$,
$C_{T^\star}(i) \in \mathrm{BR}_i(T^\star)$.
Then no destination $D \in \mathcal D_i(T^\star)$ yields a strictly higher payoff
than remaining in the current coalition.
In particular, there exists no admissible exit--and--join deviation that strictly
improves agent $i$'s payoff.
Since this holds for every agent, $T^\star$ admits no profitable admissible
exit--and--join deviation and is therefore an exit--and--join equilibrium.
\end{proof}

\begin{proposition}[Exit--and--join equilibrium as Nash equilibrium]
\label{prop:ej_nash}
A coalition structure $T^\star \in \Pi(N)$ is an exit--and--join equilibrium if and
only if the action profile
\[
D_i^\star := C_{T^\star}(i), \qquad i \in N,
\]
is a pure--strategy Nash equilibrium of the induced game $\mathcal G(T^\star)$.
\end{proposition}

\begin{proof}
Fix the coalition structure $T^\star$ and consider the induced game
$\mathcal G(T^\star)$.
By construction, the payoff of agent $i$ depends only on its own action
$D_i \in A_i(T^\star)$ and not on the actions of other agents.

An action profile $(D_i^\star)_{i\in N}$ is therefore a pure--strategy Nash
equilibrium of $\mathcal G(T^\star)$ if and only if, for every agent $i$,
\[
D_i^\star \in \arg\max_{D \in A_i(T^\star)} u_i(D;T^\star),
\]
that is, $D_i^\star \in \mathrm{BR}_i(T^\star)$.

Substituting $D_i^\star = C_{T^\star}(i)$, this condition is equivalent, by
Lemma~\ref{lem:best_response}, to $T^\star$ being an exit--and--join equilibrium.
\end{proof}

\subsection{Grand Coalition and Efficiency}

\begin{proposition}[Grand coalition stability under convexity]
\label{prop:grand_coalition_convex}
If the cooperative game $v$ is convex and switching costs are nonnegative, then
the grand coalition $\{N\}$ is an exit--and--join equilibrium under the
Aumann--Dr\`eze value.
\end{proposition}

\begin{proof}
Consider the coalition structure $T=\{N\}$.
Under the Aumann--Dr\`eze value, each agent's payoff coincides with its Shapley
value in the full game:
\[
\Omega_i(v;\{N\}) = \phi_i(v), \qquad i \in N.
\]

When $T=\{N\}$, the only exit--and--join destination available to agent $i$ is
$D=\varnothing$, which makes $i$ a singleton. The resulting payoff is
$v(\{i\})$.
For a convex cooperative game, marginal contributions are increasing with
coalition size. Therefore, for every predecessor set
$S\subseteq N\setminus\{i\}$,
\[
v(S\cup\{i\})-v(S)\ge v(\{i\})-v(\varnothing)=v(\{i\}).
\]
The Shapley value $\phi_i(v)$ is an average of these marginal contributions, so
$\phi_i(v)\ge v(\{i\})$. Nonnegative switching costs can only reduce the net
payoff from exit. Hence no agent has a strictly profitable unilateral deviation
from the grand coalition.
Hence the grand coalition $\{N\}$ is an exit--and--join equilibrium.
\end{proof}

\begin{lemma}[Efficiency of the grand coalition]
\label{lem:efficiency_grand}
If $v$ is convex, then the grand coalition $\{N\}$ maximizes total
coalition surplus
\[
V(T):=\sum_{C\in T} v(C).
\]
If, in addition, $v(A\cup B)>v(A)+v(B)$ for every pair of nonempty disjoint
coalitions $A,B\subseteq N$, then the maximizer is unique.
\end{lemma}

\begin{proof}
Convexity implies superadditivity of the cooperative game: for any disjoint
coalitions $A,B \subseteq N$,
\[
v(A \cup B) \;\ge\; v(A) + v(B).
\]

Let $T=\{C_1,\dots,C_m\}$ be an arbitrary coalition structure.
By repeated application of superadditivity,
\[
v(N)
=
v\!\left(\bigcup_{k=1}^m C_k\right)
\;\ge\;
\sum_{k=1}^m v(C_k)
=
V(T).
\]
Thus $V(T) \le V(\{N\}) = v(N)$ for every $T \in \Pi(N)$.

If the displayed strict superadditivity condition holds, then every nontrivial
partition admits a merge that strictly increases total surplus. Repeating the
merge argument gives $V(T)<V(\{N\})$ for every $T\neq\{N\}$.
\end{proof}

Convexity therefore guarantees efficiency of the grand coalition, but not
uniqueness of exit--and--join equilibrium.

\begin{example}[Nonuniqueness in a convex game]
\label{ex:nonunique_equilibria}
Let $v(S)=\sum_{i\in S} a_i$ with $a_i\in\mathbb R$.
This game is modular and hence convex. For every coalition structure $T$ and
every player $i$, the Aumann--Dr\`eze payoff is $\Omega_i(v;T)=a_i$,
independent of the coalition containing $i$.
Under zero switching costs, no exit--and--join move gives a strict improvement.
Thus every partition is an exit--and--join equilibrium. This example shows that
convexity alone does not imply uniqueness or dynamic selection of the grand
coalition.
\end{example}

\begin{remark}[Source of non-uniqueness and inefficiency]
\label{rem:inefficiency_corrected}
Non-uniqueness and inefficiency of exit--and--join equilibria arise from the
restriction to unilateral deviations combined with coalition acceptance rules,
not from a failure of convexity.
While convexity aligns total surplus with the grand coalition, it does not ensure
that surplus--improving merges can be implemented through individually rational
exit--and--join moves.
Consequently, restoring uniqueness and efficiency requires richer deviation
mechanisms, such as merge rules, coalitional deviations, or surplus redistribution
schemes that internalize the gains from coalition expansion.
\end{remark}

\section{Lyapunov Analysis of Aumann--Dr\`eze Exit--and--Join Dynamics}
\label{sec:lyapunov_ad}

This section studies the dynamic properties of the Aumann--Dr\`eze
exit--and--join process.
While the previous section characterized exit--and--join equilibria and their
game--theoretic interpretation, it did not address whether such equilibria are
reached by the induced dynamics.
Here we analyze the evolution of coalition structures under unilateral
exit--and--join moves and establish convergence properties using Lyapunov
arguments.

\subsection{Exit--and--Join Dynamics with Acceptance}

We formulate the Aumann--Dr\`eze exit--and--join process as a discrete--time,
asynchronous dynamical system on the space of coalition structures.
Algorithmic representations are introduced afterward to make the dynamics
explicit.

\subsubsection{Dynamical system formulation}

Let $\Pi(N)$ denote the finite set of coalition structures on $N$.
At each time step $t\in\mathbb{N}$, the system state is a coalition structure
$T_t \in \Pi(N)$. Given $T_t$, an agent $i\in N$ may select a destination coalition
$D\in\mathcal D_i(T_t)$ and induce the exit--and--join transition
$T_{t+1} = F_i(T_t,D)$, as defined in
Definition~\ref{def:exit_join_transition}.
Payoffs are evaluated using the Aumann--Dr\`eze value $\Omega(v;\cdot)$.

An exit--and--join action $(i,D)$ is said to be \emph{profitable} at state $T_t$ if
\begin{equation}
\label{eq:ad_improvement}
\Omega_i\bigl(v;F_i(T_t,D)\bigr)
-
c_i\bigl(T_t,F_i(T_t,D)\bigr)
>
\Omega_i(v;T_t).
\end{equation}

The destination coalition $D\in T_t$ is said to \emph{accept} agent $i$ if
\begin{equation}
\label{eq:ad_acceptance}
\Omega_\ell\bigl(v;F_i(T_t,D)\bigr)
\ge
\Omega_\ell(v;T_t),
\qquad
\forall \ell\in D.
\end{equation}

An exit--and--join transition $T_t \to F_i(T_t,D)$ is called \emph{admissible} if
both \eqref{eq:ad_improvement} and \eqref{eq:ad_acceptance} hold.
If multiple admissible actions exist, one is selected according to a specified
activation rule (deterministic or stochastic).
If no admissible action exists, the state remains unchanged.

This defines a discrete--time, asynchronous dynamical system
$T_{t+1} \in \mathcal F(T_t)$, where
$\mathcal F:\Pi(N)\rightrightarrows\Pi(N)$ is the set--valued transition map
collecting all admissible Aumann--Dr\`eze exit--and--join moves from $T_t$.

\subsubsection{Algorithmic description}

The above dynamics can be summarized procedurally by the following algorithmic
template.

\begin{algorithm}[t]
\caption{Aumann--Dr\`eze Exit--and--Join Dynamics with Acceptance}
\label{alg:ad_exit_join}
\begin{algorithmic}[1]
\State Initialize coalition structure $T \in \Pi(N)$
\State Specify an agent activation rule
\While{there exists an admissible exit--and--join action at $T$}
    \State Select an active agent $i \in N$ according to the activation rule
    \State Compute the feasible destination set $\mathcal D_i(T)$
    \State Initialize admissible destination set $\mathcal A_i \gets \varnothing$
    \For{each $D \in \mathcal D_i(T)$}
        \State Compute candidate coalition structure $T' \gets F_i(T,D)$
        \If{$\Omega_i(v;T') - c_i(T,T') > \Omega_i(v;T)$}
            \If{$\Omega_\ell(v;T') \ge \Omega_\ell(v;T)\ \forall\, \ell \in D$}
                \State $\mathcal A_i \gets \mathcal A_i \cup \{D\}$
            \EndIf
        \EndIf
    \EndFor
    \If{$\mathcal A_i \neq \varnothing$}
        \State Select a destination $D \in \mathcal A_i$
        \State Update coalition structure $T \gets F_i(T,D)$
    \EndIf
\EndWhile
\State \Return terminal coalition structure $T$
\end{algorithmic}
\end{algorithm}

The terminal coalition structure reached by the Aumann--Dr\`eze exit--and--join
dynamics is generally not unique and may depend on the initial condition.
When a strict Lyapunov or potential certificate exists, convergence occurs in
finite time, but the terminal partition is determined by the particular sequence
of admissible local improvements realized along the trajectory.
Different initial coalition structures may therefore lead to distinct
exit--and--join equilibria, each locally stable but potentially yielding
different aggregate surplus levels.

\begin{remark}[Almost sure convergence under random activation]
\label{rem:as_convergence}
When agent activations are random, convergence results are to be understood in an
almost sure sense.
In particular, suppose the dynamics admit a strict Lyapunov or potential
certificate and agents are selected according to a stochastic process with full
support on $N$ (e.g., i.i.d.\ uniform random selection). Then the exit--and--join
dynamics converge almost surely in finite time to an exit--and--join equilibrium.
Convergence relies on the eventual execution of admissible improving moves.
Random activation with full support ensures that, with probability one, any such
move is selected after finitely many steps.
Combined with a strict Lyapunov certificate and finiteness of the state space, this
guarantees almost sure termination despite the absence of centralized
coordination.
\end{remark}

\subsection{Marginal Alignment and Scalar Lyapunov Monotonicity}
\label{subsec:convex_scalar_lyapunov}

We now identify conditions under which the local incentive structure induced by
the Aumann--Dr\`eze value admits a scalar representation in terms of aggregate
coalition surplus.
Convexity of the cooperative game is important for efficiency of the grand
coalition, but it does not by itself imply that every individually profitable
exit--and--join move increases aggregate coalition surplus. A scalar Lyapunov
argument requires an additional alignment condition between individual incentives
and surplus changes.

\begin{definition}[Ordinal marginal alignment]
\label{def:ordinal_alignment}
The Aumann--Dr\`eze exit--and--join incentives satisfy \emph{ordinal marginal
alignment} with coalition surplus if, for every coalition structure
$T\in\Pi(N)$, every agent $i\in N$, and every destination
$D\in\mathcal D_i(T)$,
\begin{equation}
\label{eq:ordinal_alignment}
\begin{aligned}
&\operatorname{sign}\!\left(
\Omega_i\bigl(v;F_i(T,D)\bigr)-\Omega_i(v;T)
\right) \\
&\quad =
\operatorname{sign}\!\left(
v(D\cup\{i\})-v(D)
-
\bigl[v(C_T(i))-v(C_T(i)\setminus\{i\})\bigr]
\right).
\end{aligned}
\end{equation}
\end{definition}

\subsubsection{Coalition surplus as a scalar Lyapunov function}

Recall the coalition--surplus function
$V:\Pi(N)\to\mathbb{R}$, $V(T):=\sum_{C\in T} v(C)$.

\begin{theorem}[Scalar Lyapunov monotonicity under marginal alignment]
\label{thm:lyapunov_monotonicity}
Suppose Aumann--Dr\`eze exit--and--join incentives satisfy ordinal marginal
alignment with coalition surplus and switching costs are nonnegative.
Then, along any admissible exit--and--join transition \(T \to F_i(T,D)\), the
coalition--surplus function strictly increases:
\(V\bigl(F_i(T,D)\bigr) > V(T)\).
\end{theorem}

\begin{proof}
Let $C=C_T(i)$ be the coalition containing agent $i$ under $T$.
Only the coalitions $C$ and $D$ are affected by the transition, so
\begin{equation}
\label{eq:surplus_difference}
V\bigl(F_i(T,D)\bigr)-V(T)
=
\bigl[v(D\cup\{i\})-v(D)\bigr]
-
\bigl[v(C)-v(C\setminus\{i\})\bigr].
\end{equation}

If the transition is admissible, then
\(\Omega_i\bigl(v;F_i(T,D)\bigr)-c_i\bigl(T,F_i(T,D)\bigr)>\Omega_i(v;T)\).
Because $c_i\ge 0$, the deviating agent's Aumann--Dr\`eze payoff strictly
increases. Ordinal marginal alignment therefore implies
\(v(D\cup\{i\})-v(D)>v(C)-v(C\setminus\{i\})\).
Substituting into \eqref{eq:surplus_difference} yields
\(
V(F_i(T,D))>V(T).
\)
\end{proof}

\subsubsection{Grand coalition and efficiency}

We now formalize the relationship between convexity, efficiency, and equilibrium
under exit--and--join dynamics.

\begin{proposition}[Efficiency of the grand coalition under convexity]
\label{prop:grand_coalition_efficiency}
Suppose the cooperative game $v$ is convex.
Then the coalition--surplus function $V(T)=\sum_{C\in T} v(C)$ is maximized
over $\Pi(N)$ by the grand coalition $T^{\mathrm{gc}}=\{N\}$.
The maximizer is unique if $v(A\cup B)>v(A)+v(B)$ for every pair of nonempty
disjoint coalitions $A,B\subseteq N$.
\end{proposition}

\begin{proof}
Convexity implies superadditivity:
for any disjoint coalitions $S,T\subseteq N$, \(v(S\cup T)\ge v(S)+v(T)\).
By repeated application, any partition $T=\{C_1,\dots,C_m\}$ of $N$ satisfies
\(V(T)=\sum_{k=1}^m v(C_k)\le v(N)=V(\{N\})\),
so the grand coalition is globally efficient. Under the stated strict
superadditivity condition, every nontrivial partition can be improved by at least
one merge, which gives uniqueness.
\end{proof}

\begin{proposition}[Grand coalition as an exit--and--join equilibrium]
\label{prop:grand_coalition_equilibrium}
Suppose the cooperative game $v$ is convex and switching costs are nonnegative.
Then the grand coalition $\{N\}$ is an exit--and--join equilibrium under the
Aumann--Dr\`eze value.
\end{proposition}

\begin{proof}
This is Proposition~\ref{prop:grand_coalition_convex}. Under $T=\{N\}$, an
agent's only exit--and--join destination is the singleton option
$D=\varnothing$. Convexity implies that the Shapley value in the grand coalition
weakly dominates the singleton payoff, and nonnegative switching costs preclude a
strictly profitable exit.
\end{proof}

\begin{remark}[Efficiency versus dynamic selection]
\label{rem:efficiency_vs_selection}
Propositions~\ref{prop:grand_coalition_efficiency} and
\ref{prop:grand_coalition_equilibrium} establish that convexity guarantees both
global efficiency and stability of the grand coalition under exit--and--join
incentives. However, efficiency does \emph{not} imply uniqueness of equilibrium
or dynamic selection. The exit--and--join dynamics are governed by unilateral
deviations and acceptance constraints, and may terminate at coalition structures
that are locally stable but globally inefficient. Consequently, even under
convexity, zero switching costs, and automatic acceptance, the exit--and--join
dynamics need not converge to the grand coalition from arbitrary initial
conditions. Additional coordination mechanisms, such as merge--split deviations,
coalition formation protocols, or centralized transfers, are required to
guarantee global efficiency.
\end{remark}

\subsection{Analysis for General Cooperative Games}

We now develop the Lyapunov analysis for exit--and--join dynamics in the setting
of a general transferable--utility cooperative game.
In the absence of convexity or other surplus--alignment assumptions, scalar
Lyapunov functions need not exist.
Acceptance still gives a local monotonicity property for the deviating agent and
the destination coalition, but global convergence requires an additional
no-cycle certificate, such as a scalar Lyapunov or potential function.

\subsubsection{Restricted monotonicity under admissible transitions}

Under the acceptance rule, admissible exit--and--join transitions satisfy a
\emph{restricted monotonicity} property with respect to the Aumann--Dr\`eze
payoff vector.
This property is local to the deviating agent and the destination coalition and
does not imply Pareto improvement of the full payoff vector.

\begin{lemma}[Restricted payoff improvement]
\label{lem:restricted_monotonicity}
If $T \to F_i(T,D)$ is an admissible exit--and--join transition, then the
deviating agent $i$ strictly improves its payoff,
$\Omega_i\bigl(v;F_i(T,D)\bigr) > \Omega_i(v;T)$, and every agent
$\ell \in D$ weakly improves its payoff,
$\Omega_\ell\bigl(v;F_i(T,D)\bigr) \ge \Omega_\ell(v;T)$. No restriction is
imposed on the payoffs of agents outside $\{i\} \cup D$, including members of
the origin coalition $C_T(i)\setminus\{i\}$.
\end{lemma}

\begin{proof}
The strict improvement of agent $i$ follows from admissibility of the move.
Weak improvement of agents in $D$ follows from the acceptance rule.
No conditions are imposed on other agents.
\end{proof}

Lemma~\ref{lem:restricted_monotonicity} explicitly shows that admissible
exit--and--join transitions do \emph{not} preserve the Pareto order on the full
payoff vector. Because members of the origin coalition need not be protected, the
restricted improvement property alone does not provide a global acyclicity
argument. The next result states the certificate needed for finite termination.

\subsubsection{Absence of cycles and finite termination}

\begin{proposition}[Acyclicity under a strict Lyapunov certificate]
\label{prop:no_cycles}
Suppose there exists a function $L:\Pi(N)\to\mathbb R$ such that
$L\bigl(F_i(T,D)\bigr)>L(T)$ for every admissible exit--and--join transition
$T\to F_i(T,D)$.
Then the induced state transition graph on $\Pi(N)$ is acyclic. In particular,
every trajectory terminates in finite time at an exit--and--join equilibrium.
\end{proposition}

\begin{proof}
Suppose, by contradiction, that the dynamics admit a cycle
\(T_0 \longrightarrow T_1 \longrightarrow \cdots \longrightarrow T_K = T_0\).
The strict Lyapunov condition gives
\(L(T_0)<L(T_1)<\cdots<L(T_K)=L(T_0)\),
a contradiction. Thus no directed cycle exists. Since $\Pi(N)$ is finite, every
trajectory reaches a state with no outgoing admissible transition. By
Definition~\ref{def:exit_join_equilibrium}, such a state is an exit--and--join
equilibrium.
\end{proof}

\begin{theorem}[Finite termination under marginal alignment]
\label{thm:vector_termination}
If Aumann--Dr\`eze exit--and--join incentives satisfy ordinal marginal alignment
with coalition surplus and switching costs are nonnegative, then the
exit--and--join dynamics with acceptance terminate in finite time at a coalition
structure admitting no admissible exit--and--join deviation.
\end{theorem}

\begin{proof}
By Theorem~\ref{thm:lyapunov_monotonicity}, the coalition-surplus function
$V(T)=\sum_{C\in T}v(C)$ strictly increases along every admissible transition.
Applying Proposition~\ref{prop:no_cycles} with $L=V$ gives acyclicity and finite
termination.
\end{proof}

Thus, acceptance identifies admissible local improvements, while marginal
alignment or an exact potential structure supplies the global monotonicity needed
to guarantee convergence.

\subsubsection{Equilibrium Characterization}

\begin{proposition}[Equilibrium and Nash optimality]
\label{prop:equilibrium_nash_corrected}
Let $T^\star \in \Pi(N)$ be a coalition structure.
The following conditions are equivalent: $T^\star$ is an exit--and--join
equilibrium; the action profile \(D_i^\star := C_{T^\star}(i)\), \(i \in N\),
is a pure--strategy Nash equilibrium of the induced game $\mathcal G(T^\star)$;
and there exists no admissible exit--and--join transition
$T^\star \to F_i(T^\star,D)$ such that
\(\Omega_i\bigl(v;F_i(T^\star,D)\bigr)
-c_i\bigl(T^\star,F_i(T^\star,D)\bigr)>\Omega_i(v;T^\star)\).
\end{proposition}

\begin{proof}
\emph{(1) $\Leftrightarrow$ (2).}
By Proposition~\ref{prop:ej_nash}, $T^\star$ is an exit--and--join equilibrium if
and only if each agent's current coalition $C_{T^\star}(i)$ is a best response in
the induced game $\mathcal G(T^\star)$.
This is precisely the definition of a pure--strategy Nash equilibrium.

\medskip

\emph{(2) $\Leftrightarrow$ (3).}
A deviation $D \in \mathcal D_i(T^\star)$ is profitable in the induced game if and
only if
\(\Omega_i\bigl(v;F_i(T^\star,D)\bigr)
-c_i\bigl(T^\star,F_i(T^\star,D)\bigr)>\Omega_i(v;T^\star)\).
Hence the Nash equilibrium condition is equivalent to the absence of admissible
exit--and--join deviations that strictly improve the deviating agent's payoff,
which is exactly the definition of exit--and--join equilibrium.
\end{proof}

\begin{definition}[Exact marginal alignment]
\label{def:exact_alignment}
A cooperative game $v$ satisfies \emph{exact marginal alignment} under the
Aumann--Dr\`eze value if, for every coalition structure $T$, every agent
$i\in N$, and every destination coalition $D\in\mathcal D_i(T)$,
\[
\Omega_i\bigl(v;F_i(T,D)\bigr)-\Omega_i(v;T)
=
v(D\cup\{i\})-v(D)
-
\bigl[v(C_T(i)) - v(C_T(i)\setminus\{i\})\bigr].
\]
\end{definition}

\begin{example}[Exact marginal alignment in a linear public--good game]
\label{ex:exact_alignment_linear_public}

Let $N=\{1,\dots,n\}$ and consider the transferable--utility game
$v(S)=\sum_{i\in S} a_i+\beta\,|S|$, where
$a_i \in \mathbb{R}$ and $\beta \ge 0$.
For any $S\subseteq T\subseteq N\setminus\{i\}$,
$v(S\cup\{i\})-v(S)=a_i+\beta=v(T\cup\{i\})-v(T)$; marginal contributions are
therefore constant, and the game is convex (indeed, modular). For any coalition
$C\subseteq N$, the restricted game $v|_C$ is additive, so the Aumann--Dr\`eze
value coincides with the Shapley value and is given by
$\Omega_i(v;C)=a_i+\beta$ for $i\in C$. To verify exact marginal alignment, fix
a coalition structure $T\in\Pi(N)$, let $C=C_T(i)$, and let
$D\in\mathcal D_i(T)$.
The change in total coalition surplus induced by the exit--and--join move
$T\to F_i(T,D)$ is
\[
\begin{aligned}
&v(D\cup\{i\})-v(D)
-
\bigl[v(C)-v(C\setminus\{i\})\bigr] \\
&\qquad =
(a_i+\beta)-(a_i+\beta)
=
0.
\end{aligned}
\]
The corresponding change in agent $i$'s Aumann--Dr\`eze payoff is also zero, so
exact marginal alignment holds. Each agent contributes a private value $a_i$ and
a modular participation term
$\beta$. Because the game is additive, each agent fully internalizes its marginal
effect.
The induced exit--and--join game is therefore an exact potential game with
potential $\Phi(T)=V(T)=\sum_{C\in T} v(C)$.
\end{example}

\begin{example}[Exact potential game without convexity]
\label{ex:potential_nonconvex}

Let $N=\{1,\dots,n\}$ and define the cooperative game
$v(S)=\sum_{i\in S} a_i-\gamma \binom{|S|}{2}$, where
$a_i\in\mathbb R$ and $\gamma>0$.
The marginal contribution of agent $i$ to a coalition $S$ is
$v(S\cup\{i\})-v(S)=a_i-\gamma |S|$, which decreases with coalition size; hence
the game is submodular and is not convex in the cooperative-game sense. For any
coalition $C$ with $i\in C$, the pairwise congestion term is symmetric, and the
Aumann--Dr\`eze value is
$\Omega_i(v;C)=a_i-\frac{\gamma}{2}\bigl(|C|-1\bigr)$. Fix a coalition structure
$T$, let $C=C_T(i)$ and $D\in\mathcal D_i(T)$. Then
$\Omega_i\bigl(v;F_i(T,D)\bigr)-\Omega_i(v;T)
=\frac{\gamma}{2}\bigl(|C|-1-|D|\bigr)$.

Define
$\Phi(T):=-\frac{\gamma}{2}\sum_{C\in T}\binom{|C|}{2}$. Then
$\Phi(F_i(T,D))-\Phi(T)=\frac{\gamma}{2}\bigl(|C|-1-|D|\bigr)$.

Hence
\(\Omega_i\bigl(v;F_i(T,D)\bigr)-\Omega_i(v;T)
=\Phi(F_i(T,D))-\Phi(T)\), and the induced exit--and--join game
$\mathcal G(T)$ is an \emph{exact potential game} with potential $\Phi$, despite
the underlying cooperative game $v$ being submodular and non-convex.
\end{example}

\begin{theorem}[Potential structure and Lyapunov alignment under exact marginal alignment]
\label{thm:potential_exit_join}

Suppose the cooperative game $v$ satisfies exact marginal alignment under the
Aumann--Dr\`eze value.
Assume switching costs are zero and acceptance is automatic.

Then the induced exit--and--join game $\mathcal G(T)$ is an \emph{exact potential
game} with potential function
\(\Phi(T)=V(T):=\sum_{C\in T} v(C)\).

Moreover, the coalition--surplus function $V$ is simultaneously an exact
potential function for the induced exit--and--join game and a strict scalar
Lyapunov function for the exit--and--join dynamics.

In particular, for every admissible exit--and--join transition
$T\to T'=F_i(T,D)$,
\(u_i(D;T)-u_i(C_T(i);T)=\Phi(T')-\Phi(T)=V(T')-V(T)>0\).
\end{theorem}

\begin{proof}
Consider an admissible exit--and--join transition \(T\to T'=F_i(T,D)\).
Only two coalitions are affected by the move: the origin coalition
$C_T(i)$ and the destination coalition $D$.
Therefore, the change in coalition surplus is
\begin{equation}
\label{eq:potential_difference}
V(T')-V(T)
=
\bigl[v(D\cup\{i\})-v(D)\bigr]
-
\bigl[v(C_T(i)) - v(C_T(i)\setminus\{i\})\bigr].
\end{equation}

By exact marginal alignment, the deviating agent's Aumann--Dr\`eze payoff change
satisfies \(\Omega_i(v;T')-\Omega_i(v;T)=V(T')-V(T)\).
Since switching costs are zero, the payoff difference in the induced game
$\mathcal G(T)$ coincides with the Aumann--Dr\`eze payoff difference:
\(u_i(D;T)-u_i(C_T(i);T)=\Omega_i(v;T')-\Omega_i(v;T)\).
Combining with \eqref{eq:potential_difference} yields
\(u_i(D;T)-u_i(C_T(i);T)=V(T')-V(T)\),
establishing that $V=\Phi$ is an exact potential function for the induced
exit--and--join game.

Finally, admissibility implies that the deviating agent strictly improves its
payoff. By exact marginal alignment, this is equivalent to $V(T')>V(T)$.
Since $\Pi(N)$ is finite, $V$ is a strict Lyapunov function and guarantees
finite--time convergence of the dynamics to a local maximizer of $V$.
\end{proof}

\begin{theorem}[Exact potential structure for exit--and--join dynamics]
\label{thm:potential_without_convexity}

Consider a transferable--utility cooperative game $v$ and the associated
Aumann--Dr\`eze exit--and--join dynamics with zero switching costs and acceptance.
Suppose there exists a function \(\Phi:\Pi(N)\to\mathbb R\) such that for every
coalition structure $T\in\Pi(N)$, every agent $i\in N$, and every destination
$D\in\mathcal D_i(T)$,
\begin{equation}
\label{eq:exact_potential_condition}
\Omega_i\bigl(v;F_i(T,D)\bigr)-\Omega_i(v;T)
=
\Phi\bigl(F_i(T,D)\bigr)-\Phi(T).
\end{equation}

Then the induced exit--and--join game $\mathcal G(T)$ is an \emph{exact
potential game} with potential function $\Phi$. The function $\Phi$ also
provides an \emph{exact scalar representation of unilateral incentive changes}:
for every admissible exit--and--join transition $T\to F_i(T,D)$,
\(\Omega_i\bigl(v;F_i(T,D)\bigr)>\Omega_i(v;T)\) if and only if
\(\Phi\bigl(F_i(T,D)\bigr)>\Phi(T)\). A coalition structure $T^\star$ is an
exit--and--join equilibrium if and only if it is a pure--strategy Nash
equilibrium of the induced game $\mathcal G(T^\star)$. Along any admissible
exit--and--join transition, the potential strictly increases,
\(\Phi\bigl(F_i(T,D)\bigr)>\Phi(T)\), and the exit--and--join dynamics converge
in finite time to a coalition structure that locally maximizes $\Phi$ with
respect to accepted exit--and--join deviations.

These conclusions hold independently of convexity, superadditivity, or any
efficiency property of the cooperative game $v$.
\end{theorem}

\begin{proof}
We prove each claim in turn.

For the exact potential structure, fix a coalition structure $T\in\Pi(N)$.
In the induced exit--and--join game $\mathcal G(T)$, the payoff of agent $i$ from
choosing destination $D\in\mathcal D_i(T)$ is
\(u_i(D;T)=\Omega_i\bigl(v;F_i(T,D)\bigr)\), since switching costs are zero.
Condition~\eqref{eq:exact_potential_condition} implies that for any two
destinations $D,D'\in\mathcal D_i(T)$,
\(u_i(D';T)-u_i(D;T)=\Phi\bigl(F_i(T,D')\bigr)-\Phi\bigl(F_i(T,D)\bigr)\).
Thus unilateral payoff differences coincide exactly with differences in the
function $\Phi$, and $\mathcal G(T)$ is an exact potential game with potential
function $\Phi$.

The same identity gives a scalar representation of incentives. For any
admissible exit--and--join transition $T\to F_i(T,D)$, admissibility requires
that the deviating agent strictly improves its Aumann--Dr\`eze payoff:
\(\Omega_i\bigl(v;F_i(T,D)\bigr)>\Omega_i(v;T)\).
By~\eqref{eq:exact_potential_condition}, this inequality holds if and only if
\(\Phi\bigl(F_i(T,D)\bigr)>\Phi(T)\).
Hence $\Phi$ provides an exact scalar representation of the deviating agent's
incentive change.

For the equilibrium equivalence, recall that in an exact potential game, a
pure--strategy Nash equilibrium is equivalent to a local maximizer of the
potential function with respect to the feasible unilateral deviations encoded in
the action sets.
By part (i), this condition is equivalent to the absence of admissible
exit--and--join deviations that strictly improve the deviating agent's payoff,
which is precisely the definition of an exit--and--join equilibrium.

Finally, along every admissible exit--and--join transition, the potential $\Phi$
strictly increases.
Since the state space $\Pi(N)$ is finite, no infinite strictly increasing sequence
of potential values can exist.
Therefore, the exit--and--join dynamics converge in finite time to a coalition
structure that locally maximizes $\Phi$ with respect to accepted
exit--and--join deviations.
\end{proof}

\begin{remark}[Ordinal marginal alignment]
\label{rem:ordinal_alignment}

Exact marginal alignment (Definition~\ref{def:exact_alignment}) is a strong
structural requirement: it requires the \emph{magnitude} of an agent's
Aumann--Dr\`eze payoff change under an exit--and--join move to coincide exactly
with the induced change in total coalition surplus.
This condition is sufficient to endow the exit--and--join dynamics with an exact
potential structure, but it is not necessary for convergence or stability.
The weaker ordinal condition in Definition~\ref{def:ordinal_alignment} requires
only that individual incentives and aggregate surplus move in the same direction.
Under this condition, every admissible exit--and--join move that is profitable
for the deviating agent strictly increases total coalition surplus.
Consequently, the coalition--surplus function \(V(T)=\sum_{C\in T} v(C)\)
serves as a (strict) Lyapunov function along admissible trajectories, even
though the induced exit--and--join game need not admit an exact potential
representation.
Thus, exact marginal alignment yields an exact potential game, whereas ordinal
marginal alignment suffices for Lyapunov monotonicity, finite--time convergence,
and exclusion of cycles.
\end{remark}

\begin{remark}[Efficiency versus incentive alignment]
Cooperative convexity and dynamic incentive alignment play different roles in the
analysis. At the cooperative level, convexity of the characteristic function
$v$ means supermodularity: marginal contributions increase with coalition size.
This property implies efficiency results such as superadditivity and optimality
of the grand coalition under the Shapley or Aumann--Dr\`eze values. Cooperative
convexity governs which coalitions are socially efficient, but it is silent about
how such coalitions can be reached through decentralized behavior.
By contrast, exact and ordinal marginal alignment govern the dynamics of
individual deviations. They ensure that myopic exit--and--join moves admit a
monotone scalar representation and hence support convergence arguments, but they
do not require supermodularity of the underlying cooperative game.
These properties are logically distinct. A cooperative game may be convex without
inducing a potential structure for exit--and--join deviations, and an
exit--and--join process may admit an exact potential even when the cooperative
game is submodular. This distinction explains why efficiency and dynamic
implementability need not coincide in decentralized coalition formation.
\end{remark}

\section{Matrix Representation of the Aumann--Dr\`eze Exit--Join Dynamics}
\label{sec:ad_iterative_matrix}

The matrix formulation of the Aumann--Dr\`eze value admits a natural
iterative interpretation, which makes it particularly suitable for
dynamic exit--join coalition formation models.

\subsection{Matrix Form of the Aumann--Dr\`eze Value}

Let \(N=\{1,\dots,n\}\) be the set of players and
\(v:2^N\to\mathbb{R}\) a transferable-utility game with
\(v(\varnothing)=0\).
Fix an ordering of the nonempty coalitions of \(N\), and define the
vectorized game \(\mathrm{vec}(v)\in\mathbb{R}^{2^n-1}\). Let
\(T=\{T_1,\dots,T_m\}\) be a coalition structure on \(N\), and denote
\(t_j:=|T_j|\).
The following block representation is written in coalition-block player order;
an implicit permutation maps the resulting vector back to the canonical ordering
of agents in \(N\).
For each coalition \(T_j\), define the restricted game
\(v_j(S):=v(S)\) for all nonempty \(S\subseteq T_j\), and let
\(\mathrm{vec}(v_j)\in\mathbb{R}^{2^{t_j}-1}\) denote its vectorization. The restriction from \(v\) to \(v_j\) is implemented by a binary matrix
\[
R_{T_j}\in\{0,1\}^{(2^{t_j}-1)\times(2^n-1)},
\qquad
\mathrm{vec}(v_j)=R_{T_j}\,\mathrm{vec}(v),
\]
which selects precisely those coalition values contained in \(T_j\). Let
\(A^{(t_j)}\in\mathbb{R}^{t_j\times(2^{t_j}-1)}\) denote the Shapley value
matrix for a \(t_j\)-player game.
The Aumann--Dr\`eze value of \((v,T)\) admits the linear representation
\begin{equation}
\label{eq:ad_static}
\mathrm{AD}(v,T)
=
\begin{pmatrix}
A^{(t_1)}R_{T_1}\\
\vdots\\
A^{(t_m)}R_{T_m}
\end{pmatrix}
\mathrm{vec}(v)
=
B_T^{\mathrm{AD}}\,\mathrm{vec}(v).
\end{equation}
The operator \(B_T^{\mathrm{AD}}\) is block-diagonal across coalitions,
reflecting the locality of the Aumann--Dr\`eze value.

\subsection{Exit--Join Dynamics as a Matrix System}
\label{subsec:ad_exit_join_dynamics}

Let \(T^k\) denote the coalition structure at iteration \(k\), and define
the payoff vector
\begin{equation}
\label{eq:state}
x^k
=
\mathrm{AD}(v,T^k)
=
B_{T^k}^{\mathrm{AD}}\,\mathrm{vec}(v)
\in\mathbb{R}^n .
\end{equation}

Suppose that at iteration \(k\), an agent \(i\) exits coalition \(T_a^k\)
and joins coalition \(T_b^k\), yielding
\(T_a^{k+1}=T_a^k\setminus\{i\}\) and
\(T_b^{k+1}=T_b^k\cup\{i\}\), with all other coalitions unchanged.

Because \(B_T^{\mathrm{AD}}\) is block-diagonal, only the blocks
associated with \(T_a^k\) and \(T_b^k\) are affected. The operator update
can therefore be written as
\begin{equation}
\label{eq:operator_update}
B_{T^{k+1}}^{\mathrm{AD}}
=
B_{T^k}^{\mathrm{AD}}
+
\Delta B_a^k
+
\Delta B_b^k,
\end{equation}
where \(\Delta B_a^k\) and \(\Delta B_b^k\) replace the corresponding
coalition blocks, and all other block rows are zero.

Multiplying~\eqref{eq:operator_update} by \(\mathrm{vec}(v)\) yields the
payoff dynamics
\begin{equation}
\label{eq:payoff_update}
x^{k+1}
=
x^k
+
\left(\Delta B_a^k+\Delta B_b^k\right)\mathrm{vec}(v).
\end{equation}

Equivalently, for each player \(\ell\),
\[
x_\ell^{k+1}
=
\begin{cases}
\Phi_\ell\!\left(v|_{T_a^{k+1}}\right),
& \ell\in T_a^{k+1},\\[4pt]
\Phi_\ell\!\left(v|_{T_b^{k+1}}\right),
& \ell\in T_b^{k+1},\\[4pt]
x_\ell^{k},
& \text{otherwise},
\end{cases}
\]
where \(\Phi_\ell(w)\) denotes the Shapley value component of player
\(\ell\) in the game \(w\).

\subsection{Event--Triggered Switching Interpretation}
\label{subsec:event_triggered_ad}

The update~\eqref{eq:payoff_update} defines a discrete-time,
state-dependent switched linear system.
An exit--join move \((a,b,i)\) is \emph{admissible} if it yields a strict
payoff improvement for agent \(i\),
\begin{equation}
\label{eq:trigger}
\Phi_i\!\left(v|_{T_b^k\cup\{i\}}\right)
>
\Phi_i\!\left(v|_{T_a^k}\right).
\end{equation}

Let \(\mathcal{S}(x^k)\) denote the set of admissible exit--join moves at
state \(x^k\).
The system evolves according to
\[
x^{k+1}
=
\begin{cases}
x^k
+
\left(\Delta B_a^k+\Delta B_b^k\right)\mathrm{vec}(v),
& \text{if } \mathcal{S}(x^k)\neq\varnothing,\\[6pt]
x^k,
& \text{otherwise}.
\end{cases}
\]

If \(\mathcal{S}(x^k)=\varnothing\), the dynamics terminate at a
partition \(T^\star\) such that no agent can profitably deviate via an
exit--join move under the Aumann--Dr\`eze value.

\begin{remark}
The event-triggered formulation relies critically on the block-diagonal
structure of \(B_T^{\mathrm{AD}}\).
Values based on quotient games, such as the Owen value, do not admit a
comparable local switching representation, since a single exit alters
the entire aggregation operator.
\end{remark}

\section{Numerical Experiments}
\label{sec:numerical_experiments}

This section reports a broader numerical study of the Aumann--Dr\`eze
exit--and--join dynamics. The experiments test finite termination, monotonicity
of the coalition-surplus Lyapunov function, sensitivity to switching and
acceptance costs, and behavior in a special convex game. All figures and tables
are generated by a reproducible Python script in the repository.

We use clustered pairwise transferable-utility games with
\[
v(S)=\sum_{i\in S}a_i+\sum_{\{i,j\}\subseteq S}w_{ij}.
\]
For such games, the Aumann--Dr\`eze payoff of agent $i$ in coalition $C$ is
\[
\Omega_i(v;C)=a_i+\frac12\sum_{j\in C\setminus\{i\}}w_{ij}.
\]
The sign of the moving agent's payoff change is therefore aligned with the sign
of the coalition-surplus change. This class provides a controlled testbed for
the Lyapunov results while still allowing heterogeneous attraction and repulsion
across agents.
Weights are drawn from a clustered random model: within-cluster weights are
positive on average, while cross-cluster weights are mixed and may be negative.
Initial partitions are deliberately fragmented, active agents are sampled in
random order, and a move is executed only if it strictly improves the moving
agent's net Aumann--Dr\`eze payoff and is accepted by every member of the
destination coalition. The baseline acceptance rule uses \(\kappa=0\). In the
costly-admission experiments, a nonempty destination accepts a mover only when
each incumbent's payoff gain from admitting that mover is at least
\(\kappa\), which is \(w_{i\ell}/2\ge\kappa\) in the pairwise games.

As a representative trajectory, we use $n=30$ agents, five latent clusters, and
switching cost $c=0.05$. Starting from a fragmented partition, the process
executes 31 accepted exit--and--join moves, increases coalition surplus by
62.67, and terminates at a partition with seven coalitions. Figure
\ref{fig:representative_trajectory} shows both the membership trajectory and
the Lyapunov trace, with ten repeated terminal checks appended after the final
accepted move. The coalition labels are constant over this terminal tail, while
the scalar surplus trajectory is strictly increasing before termination and flat
after no admissible deviation remains.

\begin{figure}[t]
\centering
\begin{minipage}{0.48\linewidth}
\centering
\includegraphics[width=\linewidth]{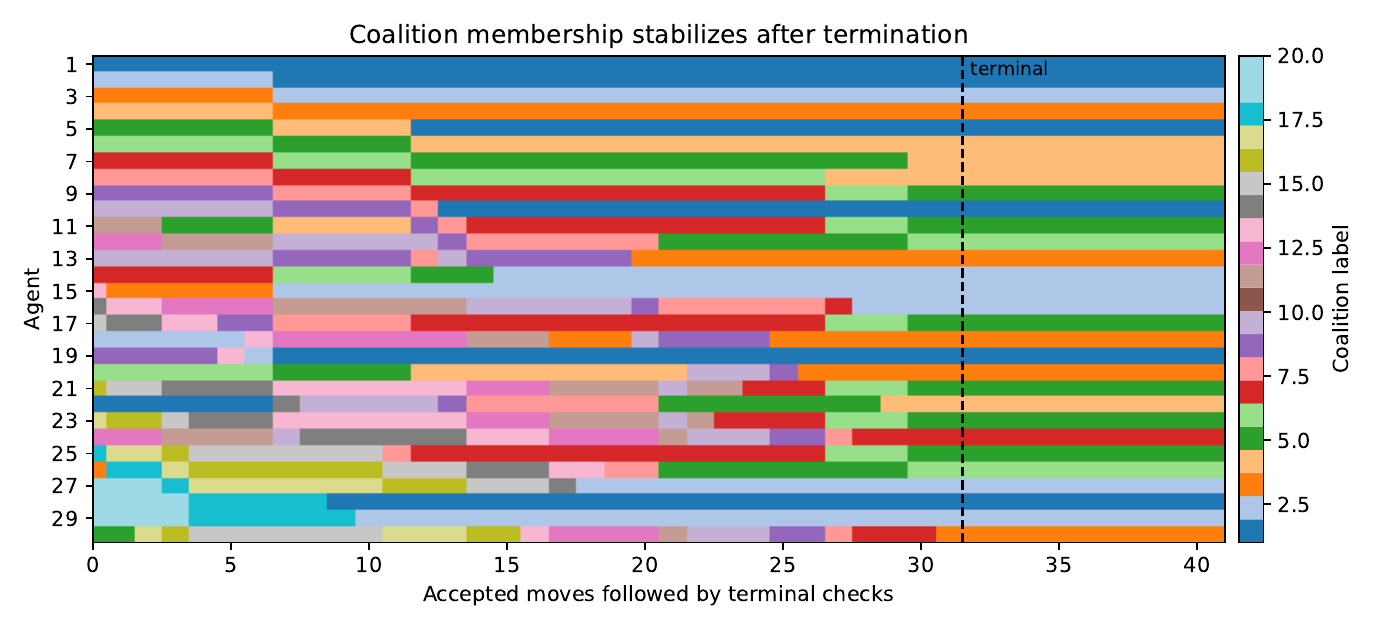}
\end{minipage}
\hfill
\begin{minipage}{0.48\linewidth}
\centering
\includegraphics[width=\linewidth]{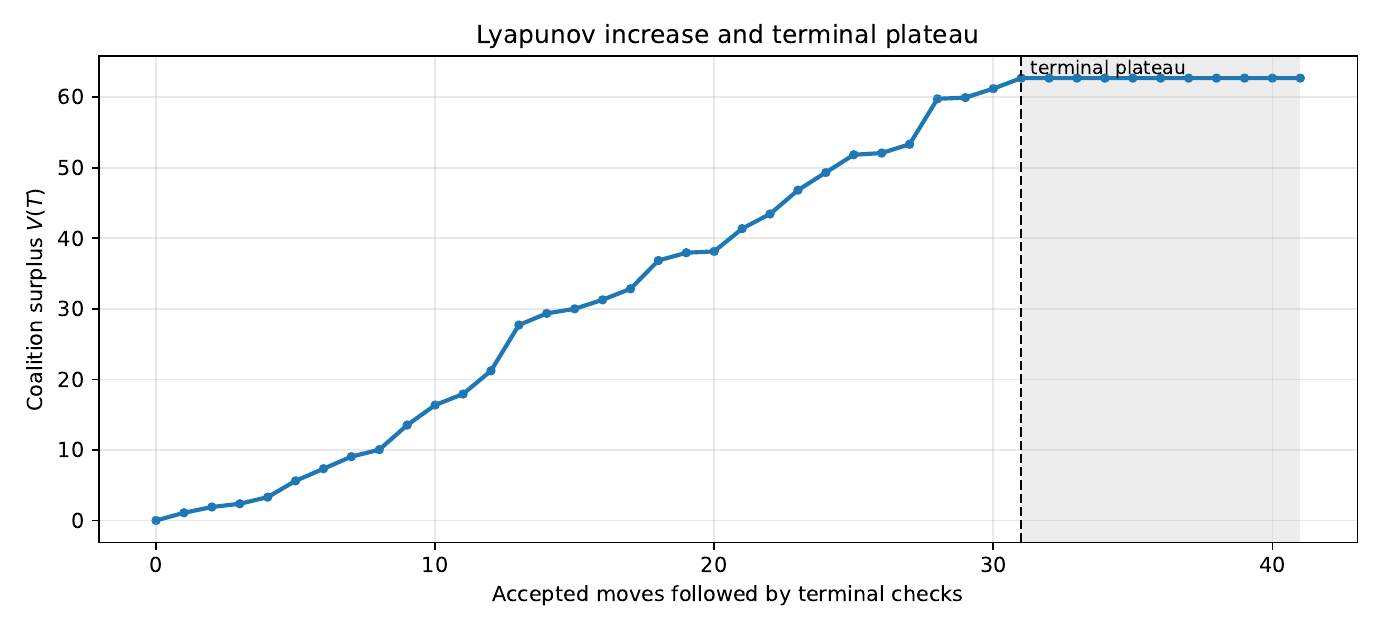}
\end{minipage}
\caption{Representative Aumann--Dr\`eze exit--and--join trajectory. Left:
coalition membership over accepted moves followed by terminal checks. Right:
coalition-surplus Lyapunov function along the same trajectory. The dashed marker
indicates the first terminal state; the flat tail makes convergence visible.}
\label{fig:representative_trajectory}
\end{figure}

To test robustness, we simulate 180 independently generated games with $n=24$
agents and four latent clusters at switching cost $c=0.05$. Every run terminates
at an exit--and--join equilibrium, and no run exhibits a violation of Lyapunov
monotonicity. The mean number of accepted moves is 21.9, the median is 22, and
the 90th percentile is 25.1. The mean surplus gain is 52.64, and terminal
partitions contain 4.9 coalitions on average. Figure~\ref{fig:monte_carlo}
summarizes the empirical distribution of termination times, surplus gains, and
terminal partition sizes.

\begin{figure}[t]
\centering
\includegraphics[width=\linewidth]{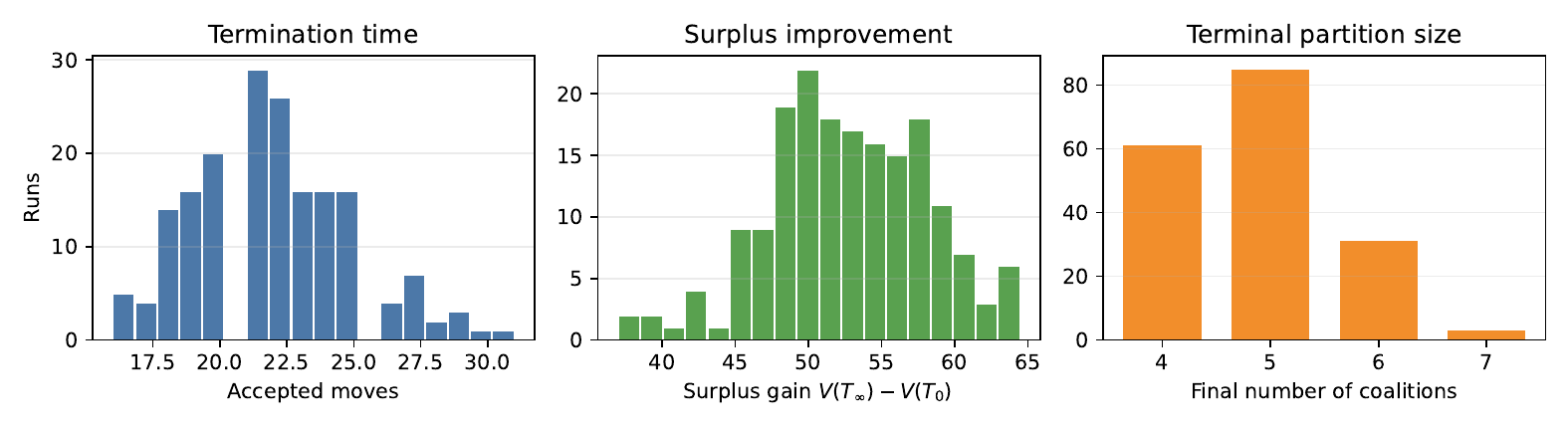}
\caption{Monte Carlo summary over 180 random clustered pairwise games. All runs
terminate with zero observed Lyapunov monotonicity violations.}
\label{fig:monte_carlo}
\end{figure}

We next vary the switching cost while holding the random game instances and
initial partitions fixed across cost levels. This paired design isolates the
effect of mobility frictions. As shown in Figure~\ref{fig:switching_cost_sweep}
and Table~\ref{tab:numerical_diagnostics}, low to moderate costs preserve most
of the surplus improvement, while high costs sharply reduce mobility and leave
the terminal partition more fragmented. At $c=1.50$, the mean number of accepted
moves falls to 2.3 and the mean terminal number of coalitions rises to 11.7,
compared with 22.9 moves and 4.9 coalitions at zero cost.

\begin{figure}[t]
\centering
\includegraphics[width=\linewidth]{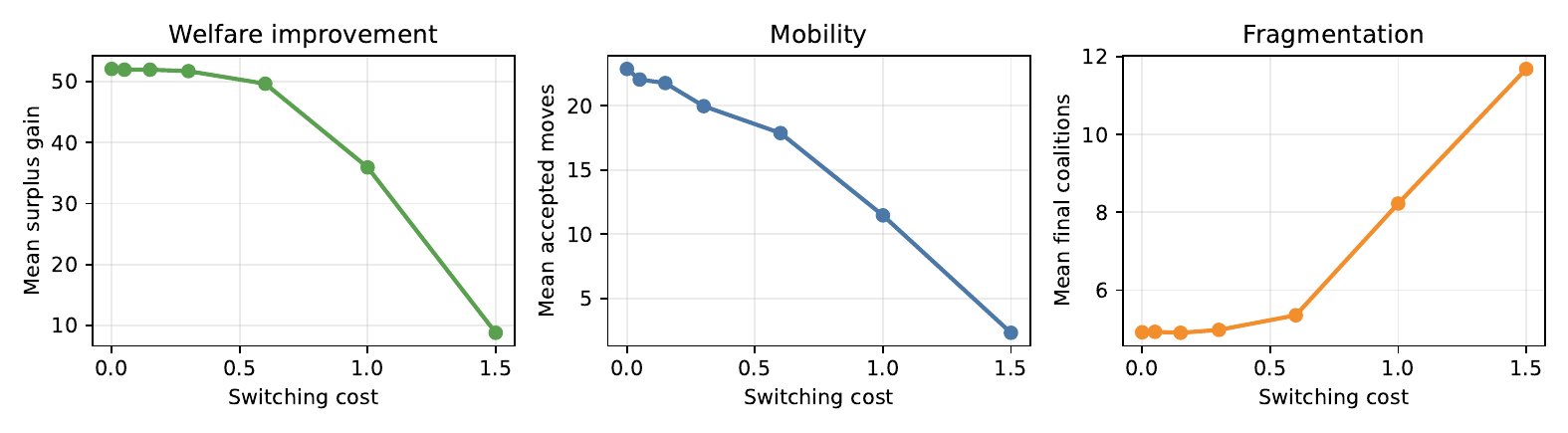}
\caption{Sensitivity to switching costs under a paired Monte Carlo design.
Higher costs reduce accepted mobility, reduce realized surplus gains, and
increase terminal fragmentation.}
\label{fig:switching_cost_sweep}
\end{figure}

We also vary the acceptance cost \(\kappa\) while holding the switching cost
fixed at \(c=0.05\) and reusing the same paired-game design. Acceptance costs
act on the destination coalition rather than on the moving agent: a higher
\(\kappa\) makes incumbents more selective about admitting a newcomer. As shown
in Figure~\ref{fig:acceptance_cost_sweep}, this admission friction has a visible
effect even though all runs still terminate and preserve Lyapunov monotonicity.
When \(\kappa\) increases from \(0\) to \(0.25\), the mean surplus gain falls
from \(53.16\) to \(35.90\), the mean number of accepted moves falls from
\(21.8\) to \(18.4\), and the mean terminal number of coalitions rises from
\(4.7\) to \(8.2\). Figure~\ref{fig:acceptance_switching_heatmap} reports the
joint effect of switching and acceptance costs; the high-cost region leaves the
process more fragmented because fewer profitable moves pass both the mover's
net-gain test and the destination's admission test.

\begin{figure}[t]
\centering
\includegraphics[width=\linewidth]{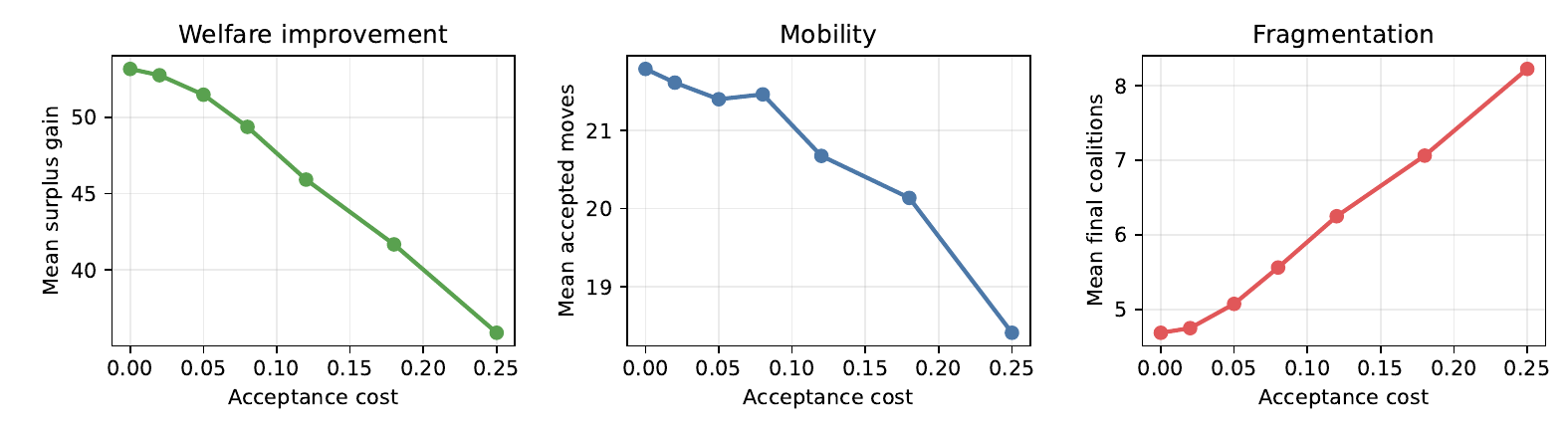}
\caption{Sensitivity to acceptance costs at switching cost \(c=0.05\). Higher
acceptance costs make destination coalitions more selective, reducing realized
surplus gains and increasing terminal fragmentation.}
\label{fig:acceptance_cost_sweep}
\end{figure}

\begin{figure}[t]
\centering
\includegraphics[width=0.92\linewidth]{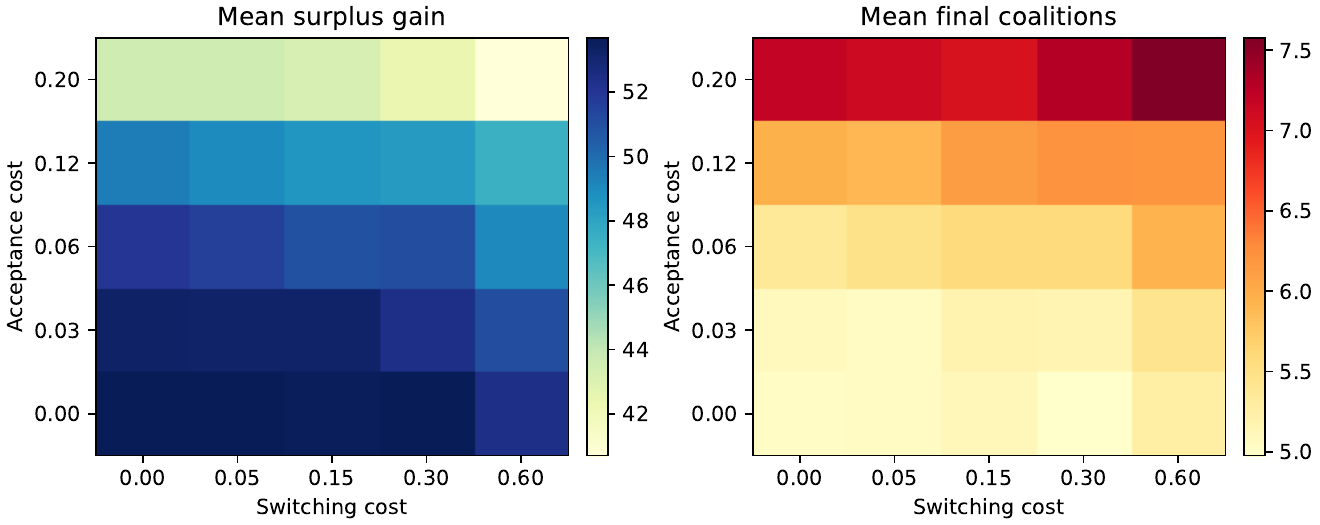}
\caption{Joint sensitivity to switching and acceptance costs under a paired
Monte Carlo design. The left panel reports mean surplus gains, and the right
panel reports mean final coalition counts.}
\label{fig:acceptance_switching_heatmap}
\end{figure}

\input{numerical_results_table.tex}

Finally, we isolate a special convex game for which the grand
coalition is both efficient and dynamically selected from a singleton initial
partition. Let \(v(S)=0.45\binom{|S|}{2}\). This symmetric game is convex
because marginal contributions increase with coalition size. With switching
cost \(c=0.05\) and acceptance cost \(\kappa=0.02\), every accepted move joins a
larger coalition, incumbents strictly prefer admission, and the process reaches
the grand coalition after 17 accepted moves. Figure~\ref{fig:convex_game_example}
shows the number of coalitions decreasing to one and the coalition-surplus
trajectory reaching the grand-coalition value \(v(N)=68.85\). This benchmark
does not overturn the earlier nonuniqueness example; rather, it shows that
strict complementarity in a convex game can make the efficient coalition
structure dynamically attractive under the exit--and--join protocol.

\begin{figure}[t]
\centering
\includegraphics[width=0.92\linewidth]{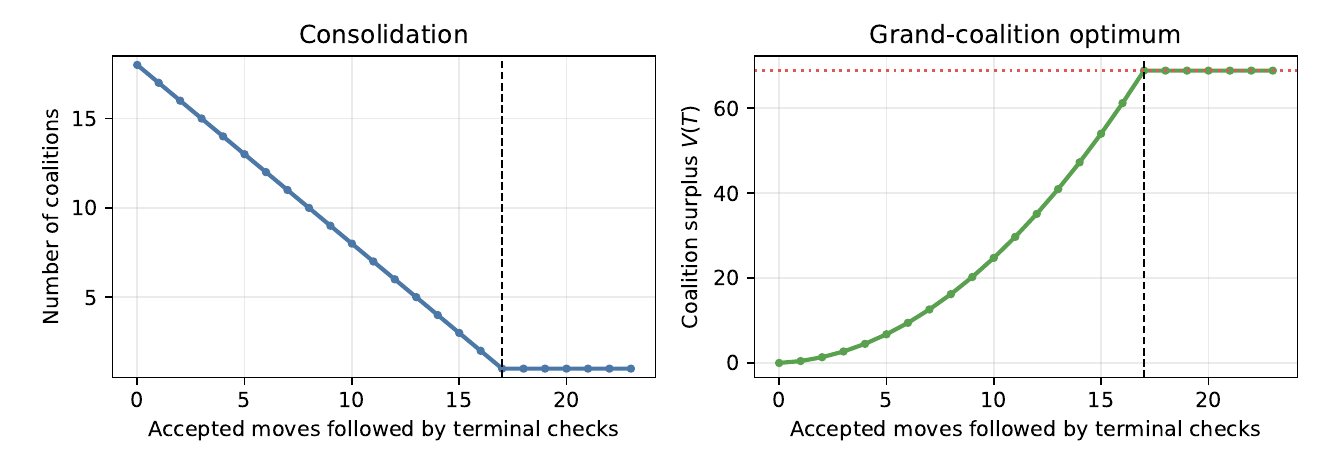}
\caption{Special convex benchmark \(v(S)=0.45\binom{|S|}{2}\). Starting from
singletons, the exit--and--join dynamics reach the grand coalition and attain
the grand-coalition surplus.}
\label{fig:convex_game_example}
\end{figure}

Taken together, the experiments support the theoretical picture. In aligned
pairwise games, accepted Aumann--Dr\`eze exit--and--join moves produce monotone
surplus improvement and finite termination across a large set of randomized
instances. Switching costs reduce the moving agent's net incentive to deviate,
while acceptance costs make destination coalitions more selective; both
frictions can stabilize more fragmented terminal partitions. The convex
benchmark complements the randomized experiments by showing a case in which
strong positive complementarities select the grand coalition through local
accepted moves.

\section{Conclusion and Extensions}
\label{sec:conclusion}

This paper developed a decentralized model of coalition formation driven by
local exit--and--join decisions. By using the Aumann--Dr\`eze value as the payoff
rule, the model preserves the cooperative interpretation of within-coalition
surplus allocation while allowing coalition structures to evolve through
noncooperative unilateral deviations. The resulting equilibrium concept is a
local stability condition: a terminal partition admits no accepted
exit--and--join move that strictly improves the deviating agent's net payoff.

The analysis separates three issues that are often conflated in coalition
formation models. Cooperative convexity supports efficiency of the grand
coalition, but does not by itself guarantee that decentralized exit--and--join
dynamics will select it. Acceptance rules provide local protection for members
of destination coalitions, but global convergence requires a no-cycle
certificate. Exact or ordinal marginal alignment supplies such a certificate by
turning unilateral incentives into a scalar potential or Lyapunov function.
The numerical experiments show that mover-side switching costs and
destination-side acceptance costs stabilize different forms of local
organization, and that a convex game with strict complementarities can select
the grand coalition through local accepted moves.

Several extensions are natural. One direction is to study strategic admission
policies in which incumbent coalitions use transfers, asymmetric sharing rules,
or switching costs to attract entrants and deter departures. Such mechanisms can
generate lock-in or monopoly-like outcomes even when the underlying surplus
function is not globally efficient. Another direction is to combine the present
model with learning, so that agents estimate coalition values from partial
observations rather than observing the relevant payoff comparisons directly.
Both directions preserve the central theme of the paper: coalition structure is
not imposed by a planner, but emerges from local incentives interacting with
coalition-level rules.

%% file: numerical_results_table.tex
\begin{table}[t]
\centering
\small
\caption{Numerical diagnostics for Aumann--Dr\`eze exit--and--join dynamics. Monte Carlo results use 180 random clustered pairwise games; the switching- and acceptance-cost sweeps use 80 paired random games per cost level.}
\label{tab:numerical_diagnostics}
\resizebox{\linewidth}{!}{%
\begin{tabular}{@{}lrrrrrr@{}}
\toprule
Experiment & Runs & Switching cost & Acceptance cost & Mean moves & Mean surplus gain & Mean final coalitions \\
\midrule
Representative trajectory & 1 & 0.05 & 0.00 & 31 & 62.67 & 7.0 \\
Monte Carlo ensemble & 180 & 0.05 & 0.00 & 21.9 & 52.64 & 4.9 \\
\midrule
Switching-cost sweep & 80 & 0.00 & 0.00 & 22.9 & 52.05 & 4.9 \\
Switching-cost sweep & 80 & 0.05 & 0.00 & 22.1 & 51.93 & 4.9 \\
Switching-cost sweep & 80 & 0.15 & 0.00 & 21.8 & 51.93 & 4.9 \\
Switching-cost sweep & 80 & 0.30 & 0.00 & 20.0 & 51.69 & 5.0 \\
Switching-cost sweep & 80 & 0.60 & 0.00 & 17.9 & 49.62 & 5.3 \\
Switching-cost sweep & 80 & 1.00 & 0.00 & 11.5 & 35.91 & 8.2 \\
Switching-cost sweep & 80 & 1.50 & 0.00 & 2.3 & 8.80 & 11.7 \\
\midrule
Acceptance-cost sweep & 80 & 0.05 & 0.00 & 21.8 & 53.16 & 4.7 \\
Acceptance-cost sweep & 80 & 0.05 & 0.02 & 21.6 & 52.75 & 4.8 \\
Acceptance-cost sweep & 80 & 0.05 & 0.05 & 21.4 & 51.47 & 5.1 \\
Acceptance-cost sweep & 80 & 0.05 & 0.08 & 21.5 & 49.36 & 5.6 \\
Acceptance-cost sweep & 80 & 0.05 & 0.12 & 20.7 & 45.92 & 6.2 \\
Acceptance-cost sweep & 80 & 0.05 & 0.18 & 20.1 & 41.68 & 7.1 \\
Acceptance-cost sweep & 80 & 0.05 & 0.25 & 18.4 & 35.90 & 8.2 \\
\midrule
Convex benchmark & 1 & 0.05 & 0.02 & 17 & 68.85 & 1.0 \\
\bottomrule
\end{tabular}
}
\begin{minipage}{0.96\linewidth}\footnotesize All Monte Carlo runs in the reported randomized sweeps terminated with zero Lyapunov monotonicity violations. The convex benchmark uses the symmetric convex game $v(S)=0.45\binom{|S|}{2}$ and reaches the grand coalition.\end{minipage}
\end{table}